\documentclass{article}
\usepackage{spconf,amsmath,graphicx}
\usepackage{booktabs}
\usepackage{cite}
\usepackage{xcolor}
\usepackage{algorithmic}
\usepackage{graphicx}
\usepackage[ruled,vlined]{algorithm2e}   
\usepackage{amssymb}
\usepackage{mathtools}
\usepackage{caption}
\usepackage{amsthm}
\usepackage{hyperref}
\usepackage{cleveref}
\newtheorem{theorem}{Theorem}[section]

\newtheorem{lemma}[theorem]{Lemma}
\newtheorem{corollary}[theorem]{Corollary}

\newtheorem{assumption}[theorem]{Assumption}

\newcommand\hz{\widehat{z}}
\newcommand\bz{\mathbf{z}}
\newcommand\hbz{\widehat{\bz}}
\newcommand\col{\operatorname{col}}

\newcommand\bE{\mathbb{E}}

\newcommand\ox{\overline{x}}
\newcommand\obx{\overline{\mathbf{x}}}

\newcommand\teta{\tilde{\eta}}

\newcommand\onbf{\overline{\nabla \bf f}}

\newcommand\nbf{\nabla \mathbf{f}}

\newcommand\md{\mathrm{D}}
\newcommand\mB{\mathrm{B}}
\newcommand\mC{\mathcal{C}}

\newcommand{\red}[1]{{\color{red}#1}}

\definecolor{purple}{rgb}{0.6, 0.2, 0.8}
   
\newcommand{\hj}[1]{{\color{magenta}#1}}
\newcommand{\zjj}[1]{{\color{teal}#1}}
\usepackage{multibib}
\title{Composite federated learning with heterogeneous data}
\name{Jiaojiao Zhang$^{\star}$,~ Jiang Hu$^\dagger$, ~Mikael Johansson$^{\star}$ \thanks{J. Zhang and M. Johansson are with the School of Electrical Engineering and Computer Science, KTH
Royal Institute of Technology, SE-100 44 Stockholm, Sweden. Email:
\{jiaoz,mikaelj\}@kth.se. J. Hu is with the Massachusetts General Hospital and Harvard Medical School, Harvard University, Boston, MA 02114 (hujiangopt@gmail.com). This work is supported in part by funding from Digital Futures and  VR under the contract 2019-05319.}}
\address{{$^{\star}$School of Electrical Engineering and Computer Science, KTH
Royal Institute of Technology}\\ $^\dagger$ Massachusetts General Hospital and Harvard Medical School, Harvard University}

\begin{document}
%
\maketitle
\begin{abstract}
We propose a novel algorithm for solving the composite Federated Learning (FL) problem.  This algorithm  manages non-smooth regularization by strategically decoupling the proximal operator and communication, and addresses client drift without any assumptions about data similarity. Moreover, each worker uses local updates to reduce the communication frequency with the server 
%
and transmits only a $d$-dimensional vector per communication round. We prove that our algorithm converges linearly to a neighborhood of the optimal solution and demonstrate the superiority of our algorithm over state-of-the-art methods in numerical experiments. 
\end{abstract}
\vspace{-2mm}
\section{Introduction}
\vspace{-2mm}
Federated Learning (FL) is a popular machine learning framework where a server coordinates a large number of workers to train a joint model without any sharing of the local data~\cite{li2020federated}. The combination of distributed computations and potential for privacy protection makes FL an attractive approach in various applications, such as machine learning~\cite{li2020federated}, wireless networks~\cite{fang2020angle}, and  Internet of Things~\cite{boobalan2022fusion}, to mention a few.

Compared to conventional distributed learning, FL suffers from a communication bottleneck at the server and is more sensitive to 
data heterogeneity among workers \cite{li2020federated,zhang2023fedaudio}.
To improve communication efficiency, McMahan et al. introduced Federated Averaging (FedAvg) \cite{mcmahan2017communication}, where the workers perform multiple local updates before updating the server with their new local states.    When data among workers is homogeneous,   the use of local updates is a practical approach for improving communication efficiency \cite{stich2018local}. 
However, when the data distribution becomes more heterogeneous, FedAvg begins to suffer from client drift. Numerous solutions have been proposed to overcome these challenges \cite{li2020federated,karimireddy2020scaffold,pathak2020fedsplit,karimireddy2020mime}.

Most existing FL algorithms focus on smooth problems. However, real-world applications often call for non-smooth objective functions,  for example, when we want to find a solution within a restricted domain or encourage  specific solution properties such as sparsity or low-rank \cite{yuan2021federated}. 
This motives us to address composite FL problems on the form 
\begin{equation}\label{eqn:basic_opt}
\begin{aligned}
    \operatorname*{minimize}_{x\in\mathbb{R}^d} \; &f(x)+g(x). 
\end{aligned}
\end{equation}
Here, $x\in \mathbb{R}^d$ is the decision vector (model parameters in a machine learning application), $f(x):= \frac{1}{n}\sum_{i=1}^nf_i(x)$ is the average data loss of the $n$ workers, and $g$ is a convex but possibly non-smooth regularizer. 
To make the data dependence explicit, we let 
$\md_i = \bigcup_{l=1}^{m_i} \md_{il}$ with $\md_{i1}, \dots, \md_{im_i}$ being the $m_i$  data points of worker $i$, 
$f_{il} (x;\md_{il})$ be the sample loss of worker $i$ associated with the data point $\md_{il}$, and
$f_i(x):=\tfrac{1}{m_i} \sum_{\md_{il}\in \md_i} f_{il} (x;\md_{il})$. 
%
Note that we do not make any assumptions about similarity between the datasets $\md_i$.
%

Solving \eqref{eqn:basic_opt} in the context of FL presents several challenges.  Federated Mirror Descent (Fed-Mid), which is a natural extension of FedAvg that replaces the local
stochastic gradient descent (SGD) steps in FedAvg with proximal SGD \cite{beck2009fast}, faces the ``curse of primal averaging'' \cite{yuan2021federated}. 
To illustrate this effect, consider the case when $g$ is the $\ell_1$-norm. Although each worker generates a sparse local model after its local updates,  averaging the local models at the server typically results in a solution that is no longer sparse.
%
Another difficulty of solving \eqref{eqn:basic_opt} in the FL setting arises due to the coupling between the proximal operator and communication. If the server averages local models that have been updated using proximal operators, it is no longer possible to directly obtain the average of the gradients across all the workers due to the nonlinearity of the general proximal operators. This makes both the algorithm design and the analysis more challenging.
\vspace{-6mm}
\subsection{Contribution} 
\vspace{-2mm}
We propose a novel algorithm for solving the composite FL problem. A key innovation of our algorithm is that it decouples  the proximal operator evaluation and the  communication to efficiently handle non-smooth regularization. Moreover, each worker uses local updates to reduce the communication frequency with the server and sends only a single $d$-dimensional vector per communication round, while addressing the issue of client drift efficiently. Without any assumptions on data similarity, we prove that our algorithm converges linearly up to a neighborhood of the optimal solution. 

\vspace{-2mm}
\subsection{Related Work}
\vspace{-1mm}
{\bf Smooth FL Problems.}
FedAvg was originally proposed in \cite{mcmahan2017communication}, and a general analysis for the homogeneous data case was carried out in \cite{stich2018local}. %
However, when data is heterogeneous, the use of local updates in FedAvg introduces client drift, which limits its practical usefulness. The client drift was analyzed theoretically  under a notion of bounded heterogeneity in~\cite{li2019convergence,zhang2021fedpd}, and  several variants of FedAvg have been proposed to reduce or eliminate the drift \cite{li2020federated,karimireddy2020scaffold,pathak2020fedsplit,karimireddy2020mime}.  
For example, SCAFFOLD \cite{karimireddy2020scaffold} and MIME \cite{karimireddy2020mime} tackle client drift by designing control variates to correct the local direction during the local updates. A drawback of these approaches is their need to communicate also the control variates, which increases the overall communication cost.  Fedsplit \cite{pathak2020fedsplit}, on the other hand, adopts the Peaceman-Rachford splitting scheme \cite{he2014strictly} to address the client drift through a consensus reformulation of the original problem where each worker only exchanges one local model per communication round. However, none of the mentioned algorithms can handle the composite FL problem.

\noindent
{\bf Composite FL Problems.} 
Compared to the abundance of studies on smooth FL problems, there are few studies for general composite problems. One attempt to address this gap is the Federated Dual Averaging (FedDA) introduced in~\cite{yuan2021federated}. In this method, each worker performs dual averaging~\cite{nesterov2009primal} during the local updates,  while the server averages the local models in the dual space and applies a proximal step. Convergence is established for general loss functions by assuming bounded gradients. However, under data heterogeneity, the convergence analysis is limited to quadratic loss functions. The Fast Federated Dual Averaging (Fast-FedDA) algorithm~\cite{bao2022fast} uses weighted summations of both past  gradient information and past model information during the local updates. However, it comes with an additional communication overhead. While the convergence of Fast-FedDA is established for general losses, it still requires the assumption of bounded heterogeneity.
The work \cite{tran2021feddr} introduces Federated Douglas-Rachford (FedDR), which avoids the assumption of bounded heterogeneity. A follow-up of FedDR is FedADMM, proposed in \cite{wang2022fedadmm}, which uses  FedDR to solve the dual problem of \eqref{eqn:basic_opt}.  
In both FedDR and FedADMM, the local updates implement an inexact evaluation of the proximal operator of the smooth loss with adaptive accuracy. However, to ensure convergence, the accuracy needs to increase by iteration, resulting in an impractically large number of local updates.

{\bf Notation.} 
We let $\|\cdot\|$ be  $\ell_2$-norm and $\|\cdot\|_1$ be $\ell_1$-norm. For positive integers $d$ 
and $n$, we let $I_d$  be the $d\times d$ identity matrix, $1_n$ 
be the all-one $n$-dimensional column vector,
and $[n] = \{1,\ldots, n\}$. We use $\otimes$ to denote the Kronecker product. For a set   $\mB$, we use $|\mB|$ to denote the cardinality.  For a  set $\mC$, we use 
$I_{\mC}(x)$ to denote the indicator function, where $I_{\mC}(x)=0$ if $x\in \mC$ and $I_{\mC}(x)=\infty$ otherwise. For a  convex function $g$, we use  $\partial g$ to denote the subdifferential. For a random variable $v$,  we use $\mathbb{E}[v]$ to denote the expectation and $\mathbb{E}[v|\mathcal{F}]$ to denote the expectation given event $\mathcal{F}$. For vectors $x_1,\ldots,x_n\in\mathbb{R}^d$, we let  $\col\{x_i\}_{i=1}^{n}=[x_1;\ldots;x_n]\in\mathbb{R}^{nd}$. Specifically, for a vector $\ox\in \mathbb{R}^d$, we let  $\col\{\ox\}_{i=1}^{n}=[\ox;\ldots;\ox] \in \mathbb{R}^{nd}$. For a vector $\omega$ and a positive scalar $\teta$, we let  $P_{\teta g}(\omega)=\arg\min_{u\in \mathbb{R}^d} \teta g(u)+\frac{1}{2}\|\omega-u \|^2 $. Specifically, for $\col\{ \omega_i\}_{i=1}^n$,  we let $P_{\teta g}(\col\{\omega_i\}_{i=1}^n)=\col\{ P_{\teta g}(\omega_i)\}_{i=1}^n$.   
\vspace{-4mm}
\section{Proposed Algorithm}
\vspace{-2mm}
The per-worker implementation of the proposed algorithm is given in Algorithm \ref{alg-fl}. 
In general, our algorithm involves communication rounds indexed by $r$ and local updates indexed by $t$. In every round $r$, workers perform $\tau$ local update steps before updating the server. 
During the local updates, each worker $i$ maintains the local model state before and after the application of the proximal mapping. We call these models pre-proximal, denoted $\hz_{i, t}^r$, and post-proximal, denoted $z_{i, t}^r$. The local mini-batch stochastic gradient with size $b$ is computed at the post-proximal local model $z_{i, t}^r$. Following this, a client-drift correction term $c_i^r=\frac{1}{\eta_g\eta\tau}( P_{\teta g}(\ox^{r-1})-{\ox^{r}})$ $- \frac{1}{\tau} \sum_{t=0}^{\tau-1}$ ${\nabla f}_i(z_{i,t}^{r-1};\mB_{i,t}^{r-1}) $ is added to the update direction for the pre-proximal local model $\hz_{i, t}^r$. 
At the end of the round, the final pre-proximal model, $ \hz_{i,\tau}^r$, is transmitted to the server.  

At the $r$-th communication round, the server also manipulates two models: a pre-proximal global model $\ox^r$ and a post-proximal global model $P_{\teta g}(\ox^r)$. 
The server calculates the average of pre-proximal local models $\hz_{i,\tau}^r$ and uses the average information to update the post-proximal global model $P_{\teta g}(\ox^r)$,  ensuring that the server-side algorithm behaves similarly to a centralized proximal SGD approach. 
Finally, the server broadcasts the pre-proximal global model $\ox^{r+1}$ to all workers that use it to update their correction terms $c_i^{r+1}$. 

As shown in Appendix~\ref{app:prf-illustration}, the proposed algorithm can be described mathematically by the following iterations
\vspace{-2mm}
\begin{equation}\label{eqn:illustration}
\hspace{-3mm}\left\{
\begin{aligned}
    \hbz_{t+1}^r
    & =\hbz_t^r\!-\!\eta \Big({\nabla \mathbf{f}}\left(\bz_t^r ; \mB_{t}^r\right)+\frac{1}{\tau} \sum_{t=0}^{\tau-1} \onbf\left(\bz_t^{{r-1}} ; \mB_{t}^{r-1} \right)\\
    & \quad- \frac{1}{\tau} \sum_{t=0}^{\tau-1} {\nabla \mathbf{f}} \left(\bz_t^{{r-1}} ; \mB_{t}^{r-1}\right)  \Big), ~\forall t\in[\tau]-1, \\
    \bz^r_{t+1}&=P_{(t+1)\eta g}\left(\hbz^{r}_{t+1} \right), ~\forall t\in[\tau]-1,\\
    \obx^{r+1}
    &= P_{\teta g} (\obx^{r})-{\eta_g}\eta \sum_{t=0}^{\tau-1} \overline{\nabla \mathbf{f}}\left(\bz_t^r ; \mB_{t}^r\right), 
\end{aligned}
\right.
\end{equation}
where  $\bz_t^r=\col\{z_{i,t}^r\}_{i=1}^n$, $\hbz_t^r=\col\{ \hz_{i,t}^r\}_{i=1}^n$, $\obx^r=\col\{\ox\}_{i=1}^{n}$,  
$\nabla \mathbf{f}(\bz_t^r;\mB_{t}^r) =\col\{\nabla f_{i}(z_{i,t}^r;\mB_{i,t}^r)\}_{i=1}^n$,  and  $ \onbf(\bz_t^r;\mB_{t}^r)$ $=\col\left\{ \frac{1}{n} {\sum_{i=1}^{n}\!\nabla f_{i}(z_{i,t}^r;\mB_{i,t}^r)}\right\}_{i=1}^n $. When $r=1$, we set ${\nabla f_i} \left(z_{i,t}^{{0}}; B_{i,t}^{0}\right) =0_d$ for all $t\in [\tau]-1$, which implies that $\frac{1}{\tau} \sum_{t=0}^{\tau-1} \onbf\left(\bz_t^{0} ; \mB_{t}^0 \right)- \frac{1}{\tau} \sum_{t=0}^{\tau-1} {\nabla \mathbf{f}} \left(\bz_t^{0} ; \mB_{t}^{0}\right)=0_{nd}$.  

Note that the updates of the post-proximal local models $\bz_{t+1}^r$ use 
the parameter 
$(t+1)\eta$ for computing the proximal operator $P_{(t+1)\eta g}{\left(\hbz_{t+1}^r\right)}$. This is similar to using a decaying step-size in stochastic gradient methods and 
significantly improves the practical performance of Algorithm \ref{alg-fl}, as we will demonstrate in numerical experiments. We give a more detailed motivation for this update in Appendix~\ref{app:line 11}. 

Our algorithm has the following additional features.

\noindent
{\bf Decoupling proximal operator evaluation and communication.}
Each worker $i$ manipulates a pre-proximal local model $\hz_{i,t}^r$ during the local updates and sends  $\hz_{i,\tau}^r$  to the server after $\tau$ local updates. The algorithm decouples proximal operator evaluation and communication in the sense that the server,  by averaging $\hz_{i,\tau}^r$,   can directly obtain the average of the local gradients across the workers, $  \sum_{t=0}^{\tau-1}\frac{1}{n} {\sum_{i=1}^{n} \!\nabla f_{i}(z_{i,t}^r;\mB_{i,t}^r)}$; cf. the last step of \eqref{eqn:illustration}. This is confirmed in the first step of \eqref{eqn:illustration}, indicating that the average correction term among all workers is zero, i.e., $ \left(\tfrac{1_n 1_n^T}{n}\!\otimes\! I_d \right) \Big(\tfrac{1}{\tau} \sum_{t=0}^{\tau-1} \onbf\left(\bz_t^{r}; \!\mB_{t}^r \right)- \frac{1}{\tau} \sum_{t=0}^{\tau-1} {\nabla \mathbf{f}} \left(\bz_t^{r} ; \!\mB_{t}^{r}\right)\Big)\!=\!0_{nd}$ for all $r\in [R]$.  In comparison, if each worker $i$ naively uses proximal SGD with client-drift correction during the local updates, i.e.,   $z_{i,t+1}^r
=P_{(t+1)\eta g} \big(z_{i,t}^r-\eta ({\nabla f}_i(z_{i,t}^r;\mB_{i,t}^r) $ $+ 	\frac{1}{\eta_g\eta\tau}( P_{\teta g}(\ox^{r-1})-\ox^{r}  )$ $-\frac{1}{\tau} \sum_{t=0}^{\tau-1} $  ${\nabla f}_i$ $(z_{i,t}^{r-1};\mB_{i,t}^{r-1}) ) \big)$,  and  sends $ z_{i,\tau}^r$ to the server after $\tau$ local updates,  then the server can no longer extract the average gradient due to the nonlinearity of proximal operator, and the resulting scheme becomes much more difficult to analyze. 

\noindent
{\bf Overcoming client drift.} From  \eqref{eqn:illustration}, we can better understand the role of the correction term. In fact, each worker $i$ utilizes $\tfrac{1}{\eta_g\eta\tau}\left( P_{\teta g}(\ox^{r-1})-\ox^{r} \right)$, 
which is equal to $\tfrac{1}{n \tau} \sum_{t=0}^{\tau-1} \sum_{i=1}^{n}$ $\nabla f_i(z_{i,t}^{r-1};\mB_{i,t}^{r-1})$,  
to introduce the local gradient information of the other workers and then replaces the previous $\tfrac{1}{\tau} \sum_{t=0}^{\tau-1} {\nabla f_i} \left(z_{i,t}^{r-1}; \mB_{i,t}^{r-1}\right)$ with the new $\nabla f_i(z_{i,t}^{r};\mB_{i,t}^{r})$. Intuitively, during the local updates, each worker $i$ approximately minimizes $\frac{1}{n}\sum_{i=1}^n f_i+g$ rather than $f_i + g $ itself, which is crucial for overcoming client drift.
Notably, only a $d$-dimensional vector is exchanged per communication round per worker, making the communication lightweight. 

\vspace{-2mm}
\begin{algorithm}[h]
\caption{Proposed Algorithm}
\label{alg-fl}
\begin{algorithmic}[1]
    \STATE $ \textbf{Input:}$ $R$, $ \tau$, $\eta$, $\eta_g$, and $\ox^{1}$
    {\STATE Set $\teta=\eta \eta_g \tau$}
   {\STATE Set $c_i^1=0_d$ for all $i\in[n]$}
    \FOR {$r = 1, 2, \ldots, R$ }
    \STATE {\bf Worker $i$}
    \STATE Set $\hz_{i, 0}^r=P_{\teta g}(\ox^r)$ and $z_{i,0}^r=P_{\teta g}(\ox^r)$
    \FOR {$ t= 0, 1, \ldots, \tau-1$ }
    \STATE Sample a subset data $\mB_{i,t}^r \subseteq \md_i $ with $|\mB_{i,t}^r |=b$
    \STATE Compute  
    $ \nabla f_i(z_{i,t}^r;\mB_{i,t}^r)=\tfrac{1}{b} \sum_{\md_{il}\in \mB_{i,t}^r} \nabla f_{il}(z_{i,t}^r;\md_{il})$
    \STATE Update 
    $
    \begin{aligned}
        &\hz_{i,t+1}^r= 
        \hz_{i,t}^r-\eta \Big({\nabla f}_i(z_{i,t}^r;\mB_{i,t}^r)+c_i^r \Big)
    \end{aligned}
    $
    \STATE Update $z_{i,t+1}^r=P_{(t+1)\eta g}{\left(\hz_{i,t+1}^r\right)}$
    \ENDFOR
    \STATE Send $ \hz_{i,\tau}^r$ to the server
    \STATE { \bf Server}
    \STATE Update $\ox^{r+1}\!=\!P_{\teta g}(\ox^r)\!+\!\eta_g( \tfrac{1}{n} \sum_{i=1}^{n} \hz_{i,\tau}^r     \!-\!P_{\teta g}(\ox^r))$ 
    \STATE Broadcast $\ox^{r+1}$  to all the workers	
    \STATE {\bf Worker $i$}
    \STATE Receive $\ox^{r+1}$ from the server
    \STATE Update $c_i^{r+1}=\frac{1}{\eta_g\eta\tau}( P_{\teta g}(\ox^{r})-\ox^{r+1})-\frac{1}{\tau} \sum_{t=0}^{\tau-1}{\nabla f}_i(z_{i,t}^{r};\mB_{i,t}^{r})$
    \ENDFOR
    \STATE \textbf{Output:} $P_{\teta g}(\ox^{R+1})$
\end{algorithmic}
\end{algorithm}

\section{ Analysis} \label{sec:analysis}

\vspace{-2mm}
In this section, we prove the convergence of Algorithm \ref{alg-fl}. All the proofs can be found in Appendix \ref{section-app}.   To facilitate the analysis, we impose the following assumptions on $f_i$ and $g$ \cite[Theorem 5.8 and Theorem 5.24]{beck2017first}.  
\vspace{-1mm}
\begin{assumption}\label{asm-convex} 
The loss function $f_i: \mathbb{R}^{d} \mapsto \mathbb{R}$ is both $\mu$-strongly convex and $L$-smooth.
\end{assumption}
\vspace{-1mm}
\begin{assumption}\label{assm:g}
The function $g: \mathbb{R}^{d} \mapsto \mathbb{R} \cup \infty $ is proper closed convex, but not necessarily smooth. In addition, we assume $g$ satisfies one of the following conditions:
    \begin{itemize}
    \vspace{-1mm}
        \item for any  $x\in \mathbb{R}^d$ and any $\widetilde{\nabla} g(x)\in \partial g(x)$, there exists a constant $0<B_g<\infty$ such that 
        $\|\widetilde{\nabla} g(x)\|\le B_g$.
         \vspace{-1mm}
        \item $g$ is an indicator function of a compact convex set.
    \end{itemize}
\end{assumption}
To handle the stochasticity caused by random sampling $ \mB_{i,t}^r$,  we denote $\mathcal{F}_t^r$ as the event generated 
by  $\{ \xi_{i,\tilde{t}}^{\tilde{r}},~ |~ i\in [n]; \tilde{r}\in [r]; \tilde{t} \in [t]-1\}$. 
We make the following assumptions regarding the stochastic gradients.

\begin{assumption}\label{asm-sgd}
The stochastic gradients of each worker $i$ satisfy
\vspace{-2mm}
\begin{equation}
    \begin{aligned}
        &\bE\left[\nabla f_i(z_{i,t}^{r}; \mB_{i,t}^r) |  \mathcal{F}_t^r\right]=\nabla f_i(z_{i,t}^{r}), \\
        &\bE \left[\left\|\nabla f_i(z_{i,t}^{r}; \mB_{i,t}^r) -\nabla f_i(z_{i,t}^{r})  \right\|^2 |  \mathcal{F}_t^r\right] \le {\sigma^2}/{b}. 
    \end{aligned}
\end{equation}
\end{assumption}
\noindent
To measure the optimality, we define 
the Lyapunov function
\vspace{-1mm}
\begin{equation}\label{eq:def-Lyapunov}
\Omega^{r}:=  \left\|P_{\teta g}(\ox^{r})-x^{\star}\right\|^2+ \|\Lambda^{r}-\overline{\Lambda}^{r}\|^2/n, 
\vspace{-2mm}
\end{equation}
where $ \Lambda^r:={\eta}( \tau\nabla \mathbf{f}({P_{\teta g}(\obx^r)})+ \sum_{t=0}^{\tau-1} \overline{\nbf}(\bz_t^{{r-1}} ; \mB_{t}^{r-1})-  \sum_{t=0}^{\tau-1} \nbf (\bz_t^{{r-1}} ; \mB_{t}^{r-1}))$,  $\overline{\Lambda}^r\!:=\!\col\left\{ \tfrac{1}{n} \sum_{i=1}^{n}\Lambda_i^r \right\}_{i=1}^n$, and $x^{\star}$ is the optimal solution to \eqref{eqn:basic_opt}. The first component $\left\|P_{\teta g}(\ox^{r})-x^{\star}\right\|^2$ in the Lyapunov function $\Omega^r$ serves to bound the optimality of the global model $P_{\teta g}(\ox^{r})$. The second component  is used to bound the client-drift error, which measures how far the local models $\{z_{i,\tau}^r\}_{i}$ are from the common initial point $P_{\teta g}(\ox^r)$ after the local updates. This drift error can be controlled by the inconsistency of the local directions accumulated through the local updates, as characterized by $\|\Lambda^{r}-\overline{\Lambda}^{r}\|^2/n$. 
We derive the following theorem.  
\vspace{-2mm}
\begin{theorem}\label{thm: g}
Under Assumptions \ref{asm-convex}, \ref{assm:g}, and 
\ref{asm-sgd},  if the step sizes satisfy 
\vspace{-2mm}
\begin{equation}\label{eq:stepsize}
    \begin{aligned}
     \teta:=\eta\eta_g \tau \le \mu/(150L^2), ~ \eta_g= \sqrt{n},
    \end{aligned}
\end{equation} 
then the sequence $\{\Omega^r\}_r$ generated by Algorithm \ref{alg-fl}   satisfies
\begin{equation*}
    \begin{aligned}
    \vspace{-2mm}
\bE\left[\Omega^{R+1}\right]\le \left(1-\frac{\mu\teta}{3}\right)^R\bE[\Omega^{1}] +\frac{30\eta\eta_g}{\mu}\frac{\sigma^2}{nb}+\frac{21\tau\eta \eta_g}{\mu n }  B_g^2.
    \end{aligned}
\end{equation*}
\end{theorem}
\noindent
Theorem \ref{thm: g} shows that  $\bE\left[\Omega^{R+1}\right]$ converges linearly to a residual error order $\mathcal{O}  ({\eta \eta_g\sigma^2}/{(\mu n b)}+ {  \tau\eta \eta_g B_g^2}/{(\mu n)} )$. The first term in the residual is controlled by the stochastic gradient variance, while the second term in the residual is due to the bound of the subgradient $\partial g$.  
%
Notably, in the special case when  $g(x)=I_{\mC}(x)$ and $\mC$ is a convex compact set, we can get rid of $B_g$ in the residual,  under the following assumption. 
\begin{assumption}\label{asmp:Ic}
When $g(x)=I_{\mC}(x)$, for the optimal solution $x^{\star}$, it holds that $\nabla f(x^{\star})=0$.
\vspace{-2mm}
\end{assumption}
Assumption \ref{asmp:Ic} is, for example, satisfied when the optimal solution $x^{\star}$ is in the interior of the convex set $\mC$. 
\begin{corollary}\label{coro-Ic}
Under Assumptions \ref{asm-convex}--\ref{asm-sgd}, and \ref{asmp:Ic},  if the step sizes satisfy  \eqref{eq:stepsize}, 
then the sequence $\{\Omega^r\}_r$ generated by Algorithm \ref{alg-fl}   satisfies
\begin{equation*}
\begin{aligned}   
\bE\left[\Omega^{R+1}\right]\le \left(1-\frac{\mu\teta}{3}\right)^R\bE[\Omega^{1}] +\frac{30\eta\eta_g}{\mu}\frac{\sigma^2}{nb}.
\end{aligned}
\end{equation*}
\end{corollary}
We will verify the theoretical results with numerical experiments in the next section. 
\vspace{-2mm}
\section{Numerical Experiments}\label{sec:experiment}
\vspace{-2mm}
Consider the sparse logistic regression problem
\begin{align}\label{eq-logreg-l1}
\underset{{x} \in \mathbb{R}^{d}}{\operatorname{minimize}} \; \frac{1}{n}\sum_{i=1}^n f_i(x)+\frac{\vartheta_2}{2}\|x\|^2+ {\vartheta_1}\|x\|_1, 
\end{align}
where $f_i(x)=\tfrac{1}{m}\sum_{l=1}^{m} \ln \left(1\!+\!\exp\left(-\left(\mathbf{a}_{i l}^{T} {x}\right) b_{i l}\right)\right)$, 
$(\mathbf{a}_{i l}, b_{i l}) $ 
 $\in \mathbb{R}^{d} \times\{-1,+1\}$ is a feature-label pair for the $l$-th sample on worker $i$, and $\vartheta_1$ and $\vartheta_2$ are the regularization parameters. 
The optimal solution $x^{\star}$ of \eqref{eq-logreg-l1} is computed in advance and the performance is measured by the optimality defined as 
$\textit{optimality}:={\|P_{\teta g}(\ox^{r})-x^{\star}\| }/{\|x^{\star}\| }.$
To generate data, we use the method in \cite{li2020federated} which allows to control the degree of heterogeneity by two parameters $(\alpha, \beta)$.

In the first set of experiments, we compare our algorithm with existing algorithms, namely FedMid\cite{yuan2021federated}, FedDA \cite{yuan2021federated}, and Fast-FedDA \cite{bao2022fast}, which all use a fixed number of local updates to solve the composite FL problem.  
We set $(\alpha, \beta)=(10, 10)$, $n=30$, $m=2000$, $\vartheta_2=0.01$, $\vartheta_1=0.0001$, and $\tau=5$. For the proposed algorithm, we use hand-tuned step sizes  $\eta=1$ and $\eta_g=1$. For FedMid and FedDA, we use the same step sizes $\eta=1$ and $\eta_g=1$. For Fast-FedDA, we use the adaptive step sizes as specified in~\cite{bao2022fast}, which are decaying step sizes. We evaluate the algorithms under both full gradients and stochastic gradients with $b=20$. 
\begin{figure}[htbp]
\begin{minipage}[h]{.49\linewidth}
    \centering
    \centerline{\includegraphics[width=4.2cm]{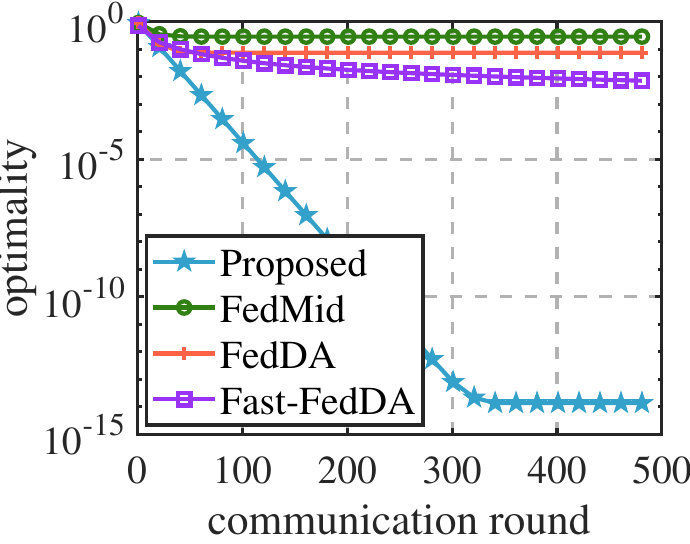}}
\end{minipage}
\hfill
\begin{minipage}[h]{0.49\linewidth}
    \centering
    \centerline{\includegraphics[width=4.2cm]{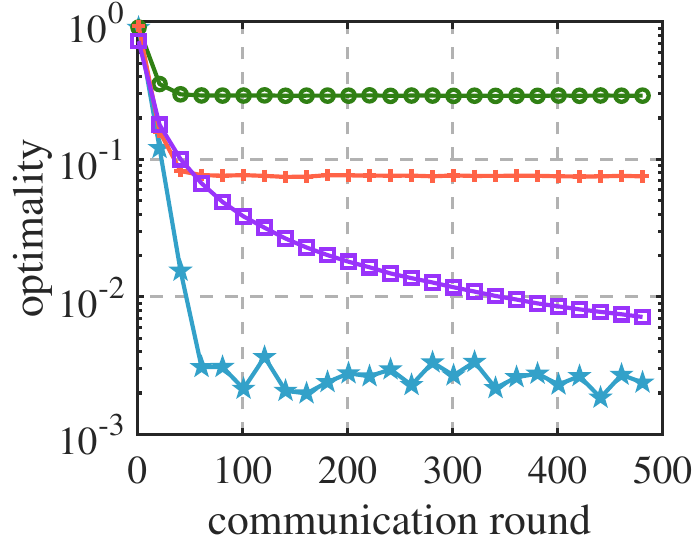}}
\end{minipage}
\vspace{-1mm}
\caption{Comparision with existing methods using full gradients (left) and stochastic gradients (right), respectively. }
\label{fig_34}
\end{figure}

As shown in Fig.~\ref{fig_34}, when using the full gradients, our algorithm achieves exact convergence. Although Theorem \ref{thm: g} suggests the existence of a residual determined by $B_g$ (the subgradient bound), our experimental results show better performance than the theoretical results, indicating that there is the possibility to improve the analysis. 

Due to client drift, FedMid and FedDA only converge to a neighborhood of the optimal solution. FedDA performs better than FedMid because it overcomes the curse of primal averaging.  Fast-FedDA converges slowly due to its decaying step sizes.
When we use stochastic gradients, our algorithm also converges to a neighborhood. The other algorithms still perform worse due to client drift or the use of decaying step sizes.

\begin{figure}[htbp]
\begin{minipage}[h]{.49\linewidth}
    \centering
    \centerline{\includegraphics[width=4.2cm]{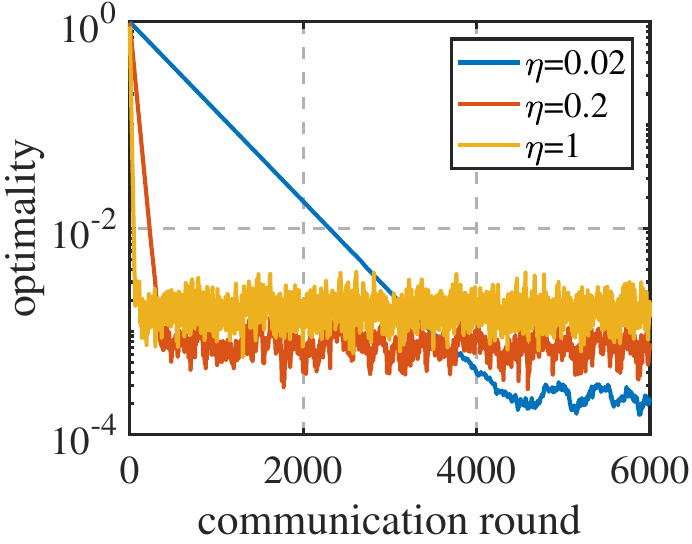}}
\end{minipage}
\hfill
\begin{minipage}[h]{0.49\linewidth}
    \centering
    \centerline{\includegraphics[width=4.2cm]{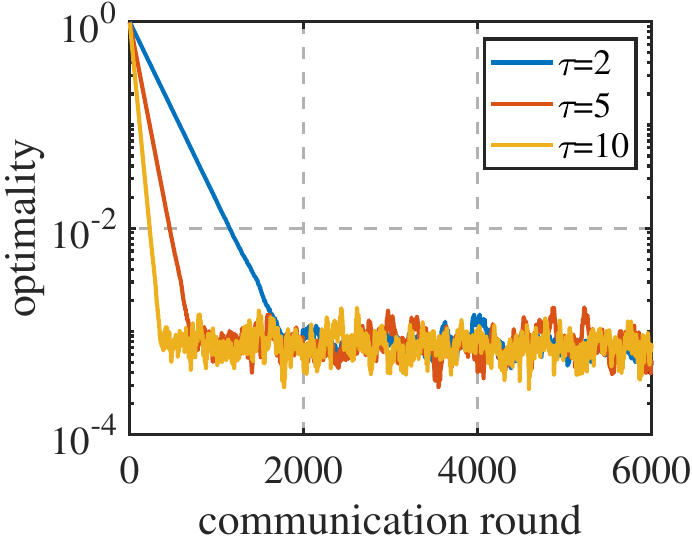}}
\end{minipage}
\vspace{-1mm}
\caption{Impact of $\eta$ (left) and $\tau$ (right) on Algorithm~\ref{alg-fl}.}
\label{fig_56}
\end{figure}
In the second set of experiments, we examine the impact of the step size $\eta$ and the number of local updates $\tau$ on our algorithm. For the impact of $\eta$, we fix $b=50$, $\eta_g=1$, and $\tau=10$ and consider $\eta\in \{ 0.02, 0.2, 1\}$. For the impact of $\tau$, on the other hand, we fix $b=50$, $\eta_g=1$, and $\eta=0.2$ and study $\tau\in \{2, 5, 10\}$.   
As shown in Fig.~\ref{fig_56},  smaller step sizes $\eta$ lead to slower convergence but higher accuracy. In addition, a larger number of local updates $\tau$ leads to faster convergence while maintaining the same level of accuracy. 
\vspace{-2mm}
\section{Conclusion}
\vspace{-2mm}
We have proposed an innovative algorithm for federated learning with composite objectives. By decoupling the proximal operator evaluation and communication, we are able to handle non-smooth regularizers in an efficient manner. The algorithm reduces the communication frequency through local updates, exchanges only a $d$-dimensional vector per communication round per worker, and addresses client drift.  We prove linear convergence up to a neighborhood of the optimal solution and show the advantages of the proposed algorithm compared to existing methods in numerical experiments. 

\newpage
\bibliographystyle{IEEEbib}
\bibliography{refs}

\newpage 
\onecolumn
\section{Appendix}\label{section-app}
\subsection{Derivation to get \eqref{eqn:illustration}}\label{app:prf-illustration}
For the sake of clarity, we provide a compact expression of the proposed Algorithm \ref{alg-fl}
\begin{equation}\label{eqn:compact}
\left\{
\begin{aligned}
    \hbz_{t+1}^r&=\hbz_t^r-\eta \left({\nabla \mathbf{f}}\left(\bz_t^r ; \mB_{t}^r\right)+\frac{1}{\eta_g\eta\tau}( P_{\teta g}(\obx^{r-1})-{\obx^{r}})- \frac{1}{\tau} \sum_{t=0}^{\tau-1} {\nabla \mathbf{f}} \left(\bz_t^{{r-1}} ; \mB_{t}^{r-1}\right)  \right), ~\forall t\in[\tau]-1, r=2,\ldots, R,\\
    \bz^r_{t+1}&= {P_{(t+1)\eta g}}\left(\hbz^{r}_{t+1} \right), ~\forall t\in[\tau]-1, r\in[R],\\
    \obx^{r+1}&=P_{\teta g}(\obx^r)+\eta_g\left(\left(\frac{1_n1_n^T}{n}\otimes I_d\right) \hbz^{r}_{\tau} - {P_{\teta g}(\obx^r)}\right), ~\forall r\in[R],
\end{aligned}
\right.
\end{equation}
where we substitute  $c_i^r$. Note that we set $ \hbz_{t+1}^1=\hbz_t^1-\eta {\nabla \mathbf{f}}\left(\bz_t^1 ; \mB_{t}^1\right)$ for all $t\in[\tau]-1$ when $r=1$. 
In the following, we derive \eqref{eqn:illustration} from \eqref{eqn:compact} by induction. To begin with, we show that \eqref{eqn:illustration} holds for $r=1$.  When $r=1$, according to \eqref{eqn:compact}, we have 
\begin{equation}\label{eq:app-11}
\begin{aligned}
    \hbz_{t+1}^1=\hbz_t^1-\eta {\nabla \mathbf{f}}\left(\bz_t^1 ; \mB_{t}^1\right), ~\forall t\in[\tau]-1. 
\end{aligned}
\end{equation}
In addition, we have $\hbz_0^1= {P_{\teta g}(\obx^1)}$ according to the initialization. After $\tau$ local updates, we get 
\begin{equation}
\begin{aligned}
    \hbz_{\tau}^1=&  {P_{\teta g}(\obx^1)}-\eta \sum_{t=0}^{\tau-1}{\nabla \mathbf{f}}\left(\bz_t^1 ; \mB_{t}^1\right). 
\end{aligned}    
\end{equation}
Substituting the above equality into the update of $\obx^{r+1}$ in \eqref{eqn:compact} and set $r=1$, we have 
\begin{equation}\label{eq-tau-11}
\begin{aligned}
    \obx^{2} 
    =&P_{\teta g}(\obx^1)+\eta_g\left(\left(\frac{1_n1_n^T}{n}\otimes I_d\right) \hbz^{1}_{\tau} - {P_{\teta g}(\obx^1)}\right) \\
    =& P_{\teta g}(\obx^1)+\eta_g\left( \left( \frac{1_n 1_n^T}{n}\otimes I_d \right)\Big(   {P_{\teta g}(\obx^1)}-\eta \sum_{t=0}^{\tau-1}{\nabla \mathbf{f}}\left(\bz_t^1 ; \mB_{t}^1  \right) \Big) - {P_{\teta g}(\obx^1)}
    \right) \\
    =& P_{\teta g} (\obx^{1})-{\eta_g}\eta \sum_{t=0}^{\tau-1} \overline{\nabla \mathbf{f}}\left(\bz_t^1 ; \mB_{t}^1\right). 
\end{aligned}
\end{equation}
After receiving $\obx^{2}$, all the workers can compute $ 
\frac{1}{\tau} \sum_{t=0}^{\tau-1} \overline{\nabla \mathbf{f}}\left(\bz_t^1 ; \mB_{t}^1\right) =\frac{P_{\teta g} (\obx^{1})-\obx^2}{ \eta_g\eta\tau } $  and use it to update the correction term to be used during the local updates of the $2$-th communication round. So far, we have shown that \eqref{eqn:illustration} holds for $r=1$.

Given that \eqref{eqn:illustration} holds for $r$ where $r\ge 2$, to establish the validity for $r + 1$,  the subsequent demonstration is intended to affirm its applicability. In accordance with the inductive hypothesis,  we have
\begin{equation}\label{eq-tau-9}
\begin{aligned}
    \hbz_{t+1}^r= & \hbz_t^r-\eta \Big({\nabla \mathbf{f}}\left(\bz_t^r ; \mB_{t}^r\right)+\frac{1}{\tau} \sum_{t=0}^{\tau-1} \onbf\left(\bz_t^{{r-1}} ; \mB_{t}^{r-1} \right)- \frac{1}{\tau} \sum_{t=0}^{\tau-1} {\nabla \mathbf{f}} \left(\bz_t^{{r-1}} ; \mB_{t}^{r-1}\right)  \Big), ~\forall t\in [\tau]-1
\end{aligned}
\end{equation}
and 
\begin{equation}\label{eq-tau-10}
\begin{aligned}
    \hbz_{\tau}^r=&  {P_{\teta g}(\obx^r)}-\eta \Big(\sum_{t=0}^{\tau-1}{\nabla \mathbf{f}}\left(\bz_t^r ; \mB_{t}^r\right)+ \sum_{t=0}^{\tau-1} \onbf\left(\bz_t^{{r-1}} ; \mB_{t}^{r-1}\right)-  \sum_{t=0}^{\tau-1} {\nabla \mathbf{f}} \left(\bz_t^{{r-1}} ; \mB_{t}^{r-1}\right)  \Big). 
\end{aligned}
\end{equation}
Substituting \eqref{eq-tau-10} into the last update in \eqref{eqn:compact}, we get
\begin{equation}
\begin{aligned}
    &\obx^{r+1}\\
    =& P_{\teta g}(\obx^r)\!+\!\eta_g\left(  \left(\frac{1_n 1_n^T}{n}\otimes I_d \right)\Big(   {P_{\teta g}(\obx^r)}\!-\!\eta \Big(\sum_{t=0}^{\tau-1}{\nabla \mathbf{f}}\left(\bz_t^r ; \mB_{t}^r\right)\!+\! \sum_{t=0}^{\tau-1} \onbf\left(\bz_t^{{r-1}} ; \mB_{t}^{r-1}\right)
    -  \sum_{t=0}^{\tau-1} {\nabla \mathbf{f}} \left(\bz_t^{{r-1}} ; \mB_{t}^{r-1}\right)  \Big)  \Big) \!-\!P_{\teta g}(\obx^r)\right) \\
    =& P_{\teta g} (\obx^{r})-{\eta_g}\eta \sum_{t=0}^{\tau-1} \overline{\nabla \mathbf{f}}\left(\bz_t^r ; \mB_{t}^r\right),  
\end{aligned}
\end{equation}
where we use $ \left(\frac{1_n 1_n^T}{n}\otimes I_d \right)\left(\sum_{t=0}^{\tau-1} \onbf\left(\bz_t^{{r-1}} ; \mB_{t}^{r-1}\right)
    -  \sum_{t=0}^{\tau-1} {\nabla \mathbf{f}} \left(\bz_t^{{r-1}} ; \mB_{t}^{r-1}\right) \right)=0_{nd}$ in the last equality. Thus, \eqref{eqn:illustration}  also holds for $r+1$. We complete the derivation to get \eqref{eqn:illustration}.  

\subsection{Motivation for employing the parameter $(t+1)\eta$ when computing $P_{(t+1)\eta g}{\left(\hz_{i,t+1}^r\right)}$ in Line 11 of Algorithm \ref{alg-fl} }\label{app:line 11}

In Line 11 of the proposed Algorithm \ref{alg-fl}, we  choose to employ the parameter $(t+1)\eta$ to compute  $P_{(t+1)\eta g}{\left(\hz_{i,t+1}^r\right)}$, where $t$ is the index of  local updates. This differs from the usage of $P_{\teta g}$ in Lines 6, 15, and 19 in Algorithm \ref{alg-fl}.
Here, we provide an explanation for the motivation and rationale behind this particular selection of the parameter $(t+1)\eta$. 

We determine this parameter $(t+1)\eta$ through testing our algorithm under the following specific setting:  considering the scenario where there is only one worker, i.e., $n=1$,  our goal is to solve $x^{\star}=\arg\min_{x\in \mathbb{R}^d}~  f_1(x)+g(x)$. In this context, this worker  employs the full gradient 
$\nabla f_1 (w)$ evaluated at a given point $w$, and we set $\tau=2$. We know that the optimal solution $x^{\star}$ satisfies $x^{\star}=P_{\beta g}(x^{\star}-\beta \nabla  f_1(x^{\star}))$ for any $\beta >0$. In this special scenario, an effective algorithm should be able to ``stop'' at the optimal solution $x^{\star}$.   This implies that when the local update commences from the optimal solution $x^{\star}$, the output of the algorithm should consistently yield $x^{\star}$.   By using the parameter 
$(t+1)\eta$ in Line 11 of Algorithm \ref{alg-fl}, our algorithm possesses this particular property. 

The detailed algorithm pseudo-code is presented in Algorithm \ref{alg-fl-appendix}, where we slightly abuse the equal sign and the assignment operator. More precisely, when the local update begins from $x^{\star}$, which implies that the initial value of the global variable $\ox^1$ is 
$x^{\star}-\teta \nabla f_1(x^{\star})$, and the output is always $x^{\star}$.  In Algorithm \ref{alg-fl-appendix}, there is no correction term since the correction term is always equal to 0 when $n=1$. 

It is important to note that in the general case when $n\ge 2$ and stochastic gradients are employed, the property of stopping at the optimal solution no longer holds. Nevertheless, our numerical experiments show that, thanks to the parameter $(t+1)\eta$, our algorithm achieves superior results compared to theoretical predictions, as illustrated in Fig.~\ref{fig_34} (left). For further details, please refer to Section \ref{sec:experiment}.  

\begin{algorithm}[h]
\caption{Proposed algorithm under the special case}
\label{alg-fl-appendix}
\begin{algorithmic}[1]
    \STATE $ \textbf{Input:}$ $R$, $\tau=2$, $\eta$, $\eta_g$
  \STATE Set $\teta=2\eta \eta_g$
    \STATE Set $\ox^{1}=x^{\star}-\teta \nabla f_1(x^{\star}) $
    \FOR {$r = 1, 2, \ldots, R$ }
    \STATE {\bf Worker $i$}
    \STATE Set $\hz_{1, 0}^r=P_{\teta g}(\ox^r)=x^{\star}$ and $z_{1,0}^r=P_{\teta g}(\ox^r)=x^{\star}$
 \STATE { $\bf t=0$}
    \STATE Update 
    $
    \begin{aligned}
        &\hz_{1,1}^r= 
        \hz_{1,0}^r-\eta {\nabla f}_1(z_{1,0}^r)=x^{\star}-\eta \nabla f_1(x^{\star})
    \end{aligned}
    $
    \STATE Update $z_{1,1}^r=P_{\eta g}{\left(\hz_{1,1}^r\right)}=P_{\eta g}{\left( x^{\star}-\eta \nabla f_1(x^{\star})   \right)}=x^{\star} $    
    \STATE { $\bf t=1$}
        \STATE Update 
    $
    \begin{aligned}
        &\hz_{1,2}^r= 
        \hz_{1,1}^r-\eta {\nabla f}_1(z_{1,1}^r)=x^{\star}-2\eta \nabla f_1(x^{\star})
    \end{aligned}
    $
    \STATE Update $z_{1,2}^r=P_{2\eta g}{\left(\hz_{1,2}^r\right)}=P_{2\eta g}{\left( x^{\star}-2\eta \nabla f_1(x^{\star})   \right)}=x^{\star} $
    \STATE Send $ \hz_{1,2}^r=x^{\star}-2\eta \nabla f_1(x^{\star})$ to the server
    \STATE { \bf Server}
    \STATE Update $\ox^{r+1}=P_{\teta g}(\ox^r)\!+\!\eta_g(  \hz_{1,2}^r     \!-\!P_{\teta g}(\ox^r))= x^{\star}+ \eta_g (x^{\star}-2\eta \nabla f_1(x^{\star})  -x^{\star}  ) =x^{\star}-\teta \nabla f_1(x^{\star})$ 
    \STATE Broadcast $\ox^{r+1}=x^{\star}-\teta \nabla f_1(x^{\star})$  to all the workers	
    \STATE {\bf Worker $i$}
    \STATE Receive $\ox^{r+1}=x^{\star}-\teta \nabla f_1(x^{\star})$ from the server
    \ENDFOR
    \STATE \textbf{Output:} $P_{\teta g}(\ox^{R+1})=x^{\star}$
\end{algorithmic}
\end{algorithm}

\subsection{Proof of Theorem \ref{thm: g}}\label{app:prf-thm-utility}
To prove Theorem \ref{thm: g}, we begin with bounding $\|\Lambda^{r} - \overline{\Lambda}^{r}\|^2$. 
\begin{lemma}\label{lem:Lambda-4.4}
Under Assumptions \ref{asm-convex}, \ref{assm:g}, and \ref{asm-sgd}, for the $\Lambda^r$ defined in \eqref{eq:def-Lyapunov}, we have 
\begin{equation}\label{eq:Lambda-19}
    \begin{aligned}
        &\bE \left\|\Lambda^{r+1} - \overline{\Lambda}^{r+1}\right\|^2-2\eta^2\tau^2 L^2 n  {\bE\left\|P_{\teta g}(\ox^{r+1})-P_{\teta g}(\ox^r)\right\|^2}\\
        \le &{4}\eta^2 \tau L^2 {\sum_{t=0}^{\tau-1} \sum_{i=1}^n\bE\left\|z_{i, t}^r-P_{\teta g}(\ox^r)\right\|^2}+ {4}\eta^2 \tau n\frac{\sigma^2}{b}.
    \end{aligned}
\end{equation}
\end{lemma}
\begin{proof}
To handle the stochastic noise in the gradients, we define  
\begin{equation*}
    \begin{aligned}
        &\mathbf{s}_t^r:=\nbf\left(\bz_t^r ; \mB_{t}^r \right)-\nabla \mathbf{f}\left(\bz_t^r\right),\\ 
        &\bar{s}_t^r:=\frac{1}{n} \sum_{i=1}^n \left(\nabla f_i\left(z_{i, t}^r ; \mB_{i,t}^r\right)-\nabla f_i\left(z_{i, t}^r\right)\right).
    \end{aligned}
\end{equation*}		
With the definition of $\Lambda^r$ in \eqref{eq:def-Lyapunov}, we have 
\begin{equation}\label{eq:Lambda-1}
    \hspace{-2mm}
    \begin{aligned}
        &\bE\left\|\Lambda^{r+1} - \overline{\Lambda}^{r+1} \right\|^2\\
        =&\eta^2\bE\left\| \tau\nabla \mathbf{f}\left(P_{\teta g}(\obx^{r+1})\right)-  \sum_{t=0}^{\tau-1} \nbf \left(\bz_t^{{r}} ; \mB_{t}^r\right)-\tau\overline{\nabla\mathbf{f}}(P_{\teta g}(\obx^{r+1}))+ \sum_{t=0}^{\tau-1} \overline{\nbf}\left(\bz_t^{{r}} ; \mB_{t}^r\right)\right\|^2 \\
        \le & \eta^2 \bE \Big\| { \tau\nabla \mathbf{f}\left(P_{\teta g}(\obx^{r+1})\right)- \tau\nabla \mathbf{f}\left(P_{\teta g}(\obx^{r})\right)}+ \tau\nabla \mathbf{f}\left(P_{\teta g}(\obx^{r})\right) - \sum_{t=0}^{\tau-1}\nabla \mathbf{f}\left({\bz}_t^{r}\right) + \sum_{t=0}^{\tau-1}\nabla \mathbf{f}\left(\bz_t^{r}\right) -  \sum_{t=0}^{\tau-1} \nbf \left(\bz_t^{{r}} ; \mB_{t}^r\right)\Big\|^2\\
        {\leq}&  2\eta^2\tau^2 L^2 n  {\bE\left\|P_{\teta g}(\ox^{r+1})-P_{\teta g}(\ox^r)\right\|^2}+ {4}\eta^2 \tau L^2 {\sum_{t=0}^{\tau-1} \sum_{i=1}^n\bE\left\|z_{i, t}^r-P_{\teta g}(\ox^r)\right\|^2}+ {4}\eta^2 \tau n\frac{\sigma^2}{b},
    \end{aligned}
\end{equation}
where  
we use $\left\| \mathbf{y} -  \left(\frac{1_n 1_n^T}{n}\otimes I_d \right)\mathbf{y}\right\|^2\le \| \mathbf{y}\|^2  $ for any vector $\mathbf{y}\in \mathbb{R}^{nd}$ in the first inequality, and $\|a_1 + a_2 \|^2 \leq 2(\|a_1\|^2 + \|a_2\|^2)$ for $a_1,a_2 \in \mathbb{R}^d$ and the fact \cite[Corollary C.1]{noble2022differentially} that 
\begin{equation}\label{eq:sigma}
\begin{aligned}
    &{\bE}\left\| \frac{1}{\tau n } \sum_{t=0}^{\tau-1} \sum_{i=1}^n   \left(\nabla f_i(z_{i, t}^r ; \mB_{i,t}^r)-\nabla f_i(z_{i, t}^r)\right)\right\|^2  \\
    = & {\bE} \!\left[ {\bE}\! \left[ \left\| {\frac{1}{ \tau n }} \sum_{t=0}^{\tau-1} \sum_{i=1}^n  \left( {\nabla f_i}\left(z_{i, t}^r ; \mB_{i,t}^r\right) \!-\! \nabla f_i\left(z_{i, t}^r\right)\right) \right\|^2 
    \!| \mathcal{F}_t^r \!\right]\! \right]  \le \frac{1}{\tau n} \frac{\sigma^2}{ b}
\end{aligned} 
\end{equation}
in the second inequality. Thus, we get \eqref{eq:Lambda-19} and complete the proof of Lemma \ref{lem:Lambda-4.4}. 
\end{proof}

In the following, we bound $\|P_{\teta g}(\ox^r)-x^\star\|^2$. 
\begin{lemma}\label{lem:x-x*-4.5}
Under Assumptions \ref{asm-convex}, \ref{assm:g}, and \ref{asm-sgd}, for the sequence $\{ P_{\teta g}(\overline{x}^r)\}_r$ generated by Algorithm \ref{alg-fl},  we have 
\begin{equation}\label{eq:x-bar-x*-24}
    \begin{aligned}
        &\bE\|P_{\teta g}(\ox^{r+1})-x^{\star} \|^2- {2\teta\left( \frac{L}{2}- \frac{1}{4\teta}\right)\bE\|P_{\teta g}(\ox^{r+1})-P_{\teta g}(\ox^{r})\|^2 }\\
        \le& \left(1-\frac{\mu\teta}{2} \right) \bE\|P_{\teta g}(\ox^{r})-x^{\star} \|^2+ 4\teta^2 \frac{1}{n\tau} \frac{\sigma^2}{b}
        +\left(4\teta^2 L^2 +\frac{2\teta L^2}{\mu}\right) \frac{1}{n\tau} \sum_{i=1}^{n} \sum_{t=0}^{\tau-1}\bE\|z_{i, t}^r-P_{\teta g}(\ox^r)\|^2. 
    \end{aligned}
\end{equation}
\end{lemma}
\begin{proof}
By  \eqref{eqn:illustration}, we have
		\begin{equation}\label{eq:bar-27}
			\begin{aligned}
				\ox^{r+1}=P_{\teta g}(\ox^r)-\frac{\teta}{n\tau} \sum_{t=0}^{\tau-1} \sum_{i=1}^n\left( \nabla f_i\left(z_{i, t}^r\right)+s_{i, t}^r\right)
			\end{aligned}
		\end{equation}
		and 
		\begin{equation}\label{eqn:pgd}
			\begin{aligned}
				P_{\teta g}(\ox^{r+1})\!=\!P_{\teta g}\Big(P_{\teta g}(\ox^r)-\frac{\teta}{n\tau} \sum_{t=0}^{\tau-1} \sum_{i=1}^n\left( \nabla f_i\left(z_{i, t}^r\right)+s_{i, t}^r\right)\Big).
			\end{aligned}
		\end{equation}
By \eqref{eqn:pgd} and \cite[Lemma 2]{zhou2018simple}, we have the so-called three-point property 
\begin{equation}\label{eqn:three-point}
    \begin{aligned}
        &\left\langle \frac{1}{n\tau} \sum_{t=0}^{\tau-1} \sum_{i=1}^n\left( \nabla f_i\left(z_{i, t}^r\right)+s_{i, t}^r\right), P_{\teta g}(\ox^{r+1})-x^{\star}\right\rangle \\
        \leq&-\frac{1}{2\teta}\left\|P_{\teta g}(\ox^{r+1})-P_{\teta g}(\ox^{r})\right\|^2+\frac{1}{2\teta}\left\|P_{\teta g}(\ox^{r})-x^{\star}\right\|^2-\frac{1}{2\teta}\left\|P_{\teta g}(\ox^{r+1})-x^{\star}\right\|^2+g\left(x^{\star}\right)-g\left(P_{\teta g}(\ox^{r+1})\right).
    \end{aligned}	
\end{equation}
Then, with the $L$-smoothness of $f$, we have 
\begin{equation}\label{eq-f+g}
    \begin{aligned}
        & g\left(P_{\teta g}(\ox^{r+1})\right)+f\left(P_{\teta g}(\ox^{r+1})\right) \\
        \leq & g\left(P_{\teta g}(\ox^{r+1})\right)+f(P_{\teta g}(\ox^r))+\left\langle\nabla f(P_{\teta g}(\ox^r)), P_{\teta g}(\ox^{r+1})-P_{\teta g}(\ox^{r})\right\rangle+\frac{L}{2}\left\|P_{\teta g}(\ox^{r+1})-P_{\teta g}(\ox^{r})\right\|^2 \\
        = & g\left(P_{\teta g}(\ox^{r+1})\right)+f(P_{\teta g}(\ox^r))+\left\langle\nabla f(P_{\teta g}(\ox^r)), x^{\star}-P_{\teta g}(\ox^{r})\right\rangle+\frac{L}{2}\left\|P_{\teta g}(\ox^{r+1})-P_{\teta g}(\ox^{r})\right\|^2 \\
        &+\left\langle\nabla f(P_{\teta g}(\ox^r)), P_{\teta g}(\ox^{r+1})-x^{\star}\right\rangle \\
        = & g\left(P_{\teta g}(\ox^{r+1})\right)+f(P_{\teta g}(\ox^r))+\left\langle\nabla f(P_{\teta g}(\ox^r)), x^{\star}-P_{\teta g}(\ox^{r})\right\rangle+\frac{L}{2}\left\|P_{\teta g}(\ox^{r+1})-P_{\teta g}(\ox^{r})\right\|^2 \\
        +&\left\langle \frac{1}{n\tau} \sum_{t=0}^{\tau-1} \sum_{i=1}^n\left( \nabla f_i\left(z_{i, t}^r\right)+s_{i, t}^r\right), P_{\teta g}(\ox^{r+1})-x^{\star} \right\rangle +\left\langle\nabla f(P_{\teta g}(\ox^r))-\frac{1}{n\tau} \sum_{t=0}^{\tau-1} \sum_{i=1}^n\left( \nabla f_i\left(z_{i, t}^r\right)+s_{i, t}^r\right), P_{\teta g}(\ox^{r+1})-x^{\star}\right\rangle  \\
        \leq & g\left(P_{\teta g}(\ox^{r+1})\right)-\frac{1}{2\teta}\left\|P_{\teta g}(\ox^{r+1})-P_{\teta g}(\ox^{r})\right\|^2+\frac{1}{2\teta}\left\|P_{\teta g}(\ox^{r})-x^{\star}\right\|^2-\frac{1}{2\teta}\left\|P_{\teta g}(\ox^{r+1})-x^{\star}\right\|^2+g\left(x^{\star}\right)-g\left(P_{\teta g}(\ox^{r+1})\right) \\
        & +f\left(x^{\star}\right)-\frac{\mu}{2}\left\|P_{\teta g}(\ox^{r})-x^{\star}\right\|^2+\frac{L}{2}\left\|P_{\teta g}(\ox^{r+1})-P_{\teta g}(\ox^{r})\right\|^2 \\
        &+\left\langle\nabla f(P_{\teta g}(\ox^r))-\frac{1}{n\tau} \sum_{t=0}^{\tau-1} \sum_{i=1}^n\left( \nabla f_i\left(z_{i, t}^r\right)+s_{i, t}^r\right), P_{\teta g}(\ox^{r+1})-x^{\star}\right\rangle\\
        = & g\left(x^{\star}\right)+f\left(x^{\star}\right)+\left(\frac{1}{2\teta}-\frac{\mu}{2}\right)\left\|P_{\teta g}(\ox^{r})-x^{\star}\right\|^2-\frac{1}{2\teta}\left\|P_{\teta g}(\ox^{r+1})-x^{\star}\right\|^2+\left(\frac{L}{2}-\frac{1}{2\teta}\right)\left\|P_{\teta g}(\ox^{r+1})-P_{\teta g}(\ox^{r})\right\|^2 \\
        &+\left\langle\nabla f(P_{\teta g}(\ox^r))-\frac{1}{n\tau} \sum_{t=0}^{\tau-1} \sum_{i=1}^n\left( \nabla f_i\left(z_{i, t}^r\right)+s_{i, t}^r\right), P_{\teta g}(\ox^{r+1})-x^{\star}\right\rangle,
    \end{aligned}
\end{equation}
where we use \eqref{eqn:three-point} and strong convexity of $f(x)$ in the second inequality.  
For the last term on the right hand of \eqref{eq-f+g}, we have 
\begin{equation}
    \begin{aligned}
        &\left\langle \nabla f(P_{\teta g}(\ox^r))- \frac{1}{n\tau} \sum_{t=0}^{\tau-1} \sum_{i=1}^n \nabla f_i\left(z_{i, t}^r\right)+{s_{i, t}^r},
        P_{\teta g}(\ox^{r+1})-x^{\star} \right\rangle\\
        =&\left\langle \nabla f(P_{\teta g}(\ox^r))- \frac{1}{n\tau} \sum_{t=0}^{\tau-1} \sum_{i=1}^n \nabla f_i\left(z_{i, t}^r\right)+{s_{i, t}^r},P_{\teta g}(\ox^{r+1})-P_{\teta g}(\ox^{r}) \right\rangle\\
        &+\left\langle \nabla f(P_{\teta g}(\ox^r))- \frac{1}{n\tau} \sum_{t=0}^{\tau-1} \sum_{i=1}^n \nabla f_i\left(z_{i, t}^r\right),
        P_{\teta g}(\ox^{r})-x^{\star} \right\rangle\\
        \le& \frac{1}{4\teta} \|P_{\teta g}(\ox^{r+1})-P_{\teta g}(\ox^{r})\|^2+\teta \|\nabla f(P_{\teta g}(\ox^r))- \frac{1}{n\tau} \sum_{t=0}^{\tau-1} \sum_{i=1}^n \nabla f_i\left(z_{i, t}^r\right)+{s_{i, t}^r} \|^2 \\
        &+\frac{L^2}{\mu} \frac{1}{n\tau} \sum_{i=1}^{n} \sum_{t=0}^{\tau-1} {\|z_{i, t}^r-P_{\teta g}(\ox^r)\|^2}+\frac{\mu}{4}\|P_{\teta g}(\ox^{r})-x^{\star}\|^2,
    \end{aligned}
\end{equation}
where we use Young's inequality in the last inequality. 
Substituting the above inequality into \eqref{eq-f+g}, we have
\begin{equation}
    \begin{aligned}
        &\bE[f(P_{\teta g}(\ox^{r+1}))+g(P_{\teta g}(\ox^{r+1}))]\\	
        \le &f(x^{\star})+g(x^{\star})+ {\left( \frac{L}{2}- \frac{1}{4\teta}\right)\bE\|P_{\teta g}(\ox^{r+1})-P_{\teta g}(\ox^{r})\|^2 }
        +\left(\frac{1}{2\teta}-\frac{\mu}{4} \right) \bE\|P_{\teta g}(\ox^{r})-x^{\star} \|^2\\&-\frac{1}{2\teta}\bE\|P_{\teta g}(\ox^{r+1})-x^{\star} \|^2
        +\left(2\teta L^2 +\frac{L^2}{\mu}\right) \frac{1}{n\tau} \sum_{i=1}^{n} \sum_{t=0}^{\tau-1} \bE\|z_{i, t}^r-P_{\teta g}(\ox^r)\|^2+ 2\teta \frac{1}{n\tau} \frac{\sigma^2}{b}, 
    \end{aligned}
\end{equation}
where we use \eqref{eq:sigma}. 
Thus, we have
\begin{equation}
    \begin{aligned}
        &\frac{1}{2\teta}\bE\|P_{\teta g}(\ox^{r+1})-x^{\star} \|^2\\
        \le& \left(\frac{1}{2\teta}-\frac{\mu}{4} \right) \bE\|P_{\teta g}(\ox^{r})-x^{\star} \|^2
        +\left(2\teta L^2 +\frac{L^2}{\mu}\right) \frac{1}{n\tau} \sum_{i=1}^{n} \sum_{t=0}^{\tau-1}\bE\|z_{i, t}^r-P_{\teta g}(\ox^r)\|^2+ 2\teta \frac{1}{n\tau} \frac{\sigma^2}{b}\\
        & {+\left( \frac{L}{2}- \frac{1}{4\teta}\right)\bE\|P_{\teta g}(\ox^{r+1})-P_{\teta g}(\ox^{r})\|^2 },
    \end{aligned}
\end{equation}  
which implies that
\begin{equation}
    \begin{aligned}
        &\bE\|P_{\teta g}(\ox^{r+1})-x^{\star} \|^2\\
        \le& \left(1-\frac{\mu\teta}{2} \right) \bE\|P_{\teta g}(\ox^{r})-x^{\star} \|^2 	+ 4\teta^2 \frac{1}{n\tau} \frac{\sigma^2}{b}
        +\left(4\teta^2 L^2 +\frac{2\teta L^2}{\mu}\right) \frac{1}{n\tau} \sum_{i=1}^{n} \sum_{t=0}^{\tau-1} {\bE\|z_{i, t}^r-P_{\teta g}(\ox^r)\|^2}\\
        & {+2\teta\left( \frac{L}{2}- \frac{1}{4\teta}\right)\bE\|P_{\teta g}(\ox^{r+1})-P_{\teta g}(\ox^{r})\|^2 }. 
    \end{aligned}
\end{equation}
We get \eqref{eq:x-bar-x*-24} and complete the proof of Lemma \ref{lem:x-x*-4.5}. 

\end{proof}
In the following, we bound  $\sum_{t=0}^{\tau-1} \sum_{i=1}^n \mathbb{E}\left\|\hz_{i, t}^r-P_{\teta g}(\ox^r)\right\|^2$,   
which will be used to bound  $\sum_{t=0}^{\tau-1} \sum_{i=1}^n \mathbb{E}\left\|z_{i, t}^r-P_{\teta g}(\ox^r)\right\|^2$ that appears on both right hands of \eqref{eq:Lambda-19}  and \eqref{eq:x-bar-x*-24}.  The proof follows similar steps as in \cite{karimireddy2020scaffold,alghunaim2023local}.  
\begin{lemma}\label{lem:phi-bar-x}
Under Assumptions \ref{asm-convex},  \ref{assm:g}, and \ref{asm-sgd}, if $\eta \leq \frac{1}{{4\sqrt{2}} L \tau}$, we have
\begin{equation}\label{eq:phi--29}
    \begin{aligned}
        &{\mathbb{E}\left[\sum_{t=0}^{\tau-1} \sum_{i=1}^n\left\|\hz_{i, t}^r- {P_{\teta g}(\ox^r)}\right\|^2\right]}
        \le	
        { {32} \tau\mathbb{E}\left\|\Lambda^r-\overline{\Lambda}^r\right\|^2 }+{32} \eta^2\tau^3 n  \mathbb{E}\left\| {\nabla f}(P_{\teta g}(\ox^r)) \right\|^2+{8} \eta^2 \tau^2 n\frac{\sigma^2}{b}+n\tau  {\frac{\teta^2}{\eta_g^2}} B_g^2. 
    \end{aligned}
\end{equation}
\end{lemma}
\begin{proof}	
If $\tau=1$, then $\hz_{i, t}^r-P_{\teta g}(\ox^r) =0_{d} $. 
Now suppose that $\tau \geq 2$. According to \eqref{eqn:illustration}, we have 
\begin{equation}\label{eq-phi-P1}
    \begin{aligned}
        &\mathbb{E}\left[ \sum_{i=1}^n\left\|\hz_{i, t+1}^r- {P_{\teta g}(\ox^r)}\right\|^2\right]\\
        =&\mathbb{E}\Big\|\hbz_t^r- {P_{\teta g}(\obx^r)}-\eta \big(\nbf\left(\bz_t^r ; \mB_{t}^r \right)-\nabla \mathbf{f}\left(P_{\teta g}(\obx^r)\right) + {\nabla \mathbf{f}\left(P_{\teta g}(\obx^r) \right)} {+\frac{1}{\tau} \sum_{t=0}^{\tau-1} \overline{\nbf}\left(\bz_t^{{r-1}} ; \mB_{t}^{r-1}\right)- \frac{1}{\tau} \sum_{t=0}^{\tau-1} \nbf \left(\bz_t^{{r-1}} ; \mB_{t}^{r-1}\right)  \big)} \Big\|^2 \\
        {\leq}&\left(1+\frac{1}{\tau-1}\right) \mathbb{E}\left\|\hbz_t^r- {P_{\teta g}(\obx^r)}\right\|^2+\frac{{2 }\tau\eta^2 n}{\tau}\frac{\sigma^2}{b}+{2 }\tau \mathbb{E}\left\|\eta \nabla \mathbf{f}\left(\bz_t^r \right)-\eta\nabla \mathbf{f}\left(P_{\teta g}(\obx^r)\right) +\frac{{\Lambda^r}}{\tau} \right\|^2 \\
        {\leq}&\left(1+\frac{1}{\tau-1}\right) \bE\left\|\hbz_t^r- {P_{\teta g}(\obx^r)}\right\|^2+ {2}\eta^2n \frac{\sigma^2}{b}+ {4} \eta^2 \tau \mathbb{E}\left\| \nabla \mathbf{f}\left(\bz_t^r \right)-\nabla \mathbf{f}\left(P_{\teta g}(\obx^r)\right)\right\|^2+  \frac{{4}}{\tau}\bE\left\|\Lambda^r\right\|^2\\
        \leq&\left(\!1+\frac{5 / 4}{\tau\!-\!1}\!\right) \mathbb{E}\left\|\hbz_t^r\!-\! {P_{\teta g}(\obx^r)} \right\|^2\!+\!\frac{ {4}\bE\left\|\Lambda^r\right\|^2 }{\tau }+ {2}\!\eta^2 n \frac{\sigma^2}{b}+\frac{\teta^2}{4\tau\eta_g^2} nB_g^2,
    \end{aligned}
\end{equation}
where we use $(a_1+a_2)^2\le (1/\theta) a_1^2 + (1/(1-\theta)) a_2^2$  with $\theta=1-1/\tau$ in the first inequality  and use {$\eta \leq \frac{1}{{4\sqrt{2}} L \tau }$ } and 
\begin{equation}\label{eq-Bg}
    \begin{aligned}
        &{4} \eta^2 \tau \left\| \nabla \mathbf{f}\left(\bz_t^r \right)-\nabla \mathbf{f}\left(P_{\teta g}(\obx^r)\right)\right\|^2\\
        \le	&{4} \eta^2 \tau L^2 \left\|\bz_t^r- {P_{\teta g}(\obx^r)}\right\|^2 \\
        =&{4} \eta^2 \tau  L^2(2\|  {P_{t\eta g}(\hbz_t^r)} -\hbz_t^r \|^2+2\| \hbz_t^r - {P_{t \eta g}(\obx^r)}\|^2)\\
        \le&{4} \eta^2 \tau  L^2\left( 2 {t^2 \eta^2} \underbrace{\|\widetilde{\nabla }g( P_{t\eta g}(\hbz_t^r)
            )\|^2}_{\le n B_g^2} + 2\left\|\hbz_t^r- {P_{\teta g}(\obx^r)}\right\|^2\right)\\
        \le& \frac{1}{4\tau}  \left\|\hbz_t^r- {P_{\teta g}(\obx^r)}\right\|^2+ \frac{1}{4\tau}   {\frac{\teta^2}{\eta_g^2}} nB_g^2
    \end{aligned}
\end{equation}
in the last inequality. Telescoping  \eqref{eq-phi-P1}, we have
\begin{equation}\label{eq:phi-31}
    \begin{aligned}
        &\mathbb{E}\left[ \sum_{i=1}^n\left\|\hz_{i, t+1}^r- {P_{\teta g}(\ox^r)}\right\|^2\right]\\
        \leq &\left(\frac{ {4}}{\tau} \mathbb{E}\left\|\Lambda^r\right\|^2+{ {2}\eta^2 n} \frac{\sigma^2}{b}+\frac{\teta^2}{4\tau  {\eta_g^2}}n B_g^2\right) \sum_{\ell=0}^{t}\left(1+\frac{5 / 4}{(\tau-1)}\right)^{\ell} \\
        \leq& \left(\frac{ {4}}{\tau} \mathbb{E}\left\|\Lambda^r\right\|^2+{{2}\eta^2 n} \frac{\sigma^2}{b}+\frac{\teta^2}{4\tau {\eta_g^2}}n B_g^2\right) \sum_{\ell=0}^{t} \exp \left(\frac{(5 / 4)\ell}{\tau-1}\right) \\
        \le& {16}\mathbb{E}\left\|\Lambda^r\right\|^2+ {8}\eta^2 \tau n \frac{\sigma^2}{b}+ {\frac{\teta^2}{\eta_g^2}}nB_g^2,
    \end{aligned}
\end{equation}
where  we use $\sum_{\ell=0}^{t}\left(1+\frac{5/4}{\tau-1}\right)^\ell \leq\sum_{\ell=0}^{t} \exp \left(\frac{5/4 \ell}{\tau-1}\right) \leq\sum_{\ell=0}^{t} \exp (5/4) \le 4\tau$ for $\ell \leq \tau-1$. 
By summing \eqref{eq:phi-31} over  $t=0,\ldots,\tau-1$, we obtain 
\begin{equation}\label{eq:phi-32}
    \begin{aligned}
        &{\mathbb{E}\left[\sum_{t=0}^{\tau-1} \sum_{i=1}^n\left\|\hz_{i, t}^r- {P_{\teta g}(\ox^r)}\right\|^2\right]}\\
        \leq& \tau \left(  {16}\mathbb{E}\left\|\Lambda^r\right\|^2+ {8}\eta^2 \tau n \frac{\sigma^2}{b}+\frac{\teta^2}{\eta_g^2}nB_g^2\right) \\
        \le	&
        { {32} \tau\mathbb{E}\left\|\Lambda^r-\overline{\Lambda}^r\right\|^2 }+{32} \eta^2\tau^3 n  \mathbb{E}\left\| {\nabla f}(P_{\teta g}(\ox^r)) \right\|^2+{8} \eta^2 \tau^2 n\frac{\sigma^2}{b}+n\tau  {\frac{\teta^2}{\eta_g^2}} B_g^2, 
    \end{aligned}
\end{equation}
where we  use $\overline{\Lambda}^r= \eta\tau \overline{\nabla \mathbf{f}}\left(P_{\teta g}(\obx^r)\right)$ in the last inequality. 
This completes the proof of Lemma \ref{lem:phi-bar-x}.
\end{proof}
By substituting  \eqref{eq:phi--29}  into  \eqref{eq:Lambda-19}, we obtain 
\begin{equation}\label{eq:Lambda-34}
\begin{aligned}
    &\frac{1}{n}\bE\left\|\Lambda^{r+1}-\overline{\Lambda}^{r+1}\right\|^2-2\eta^2\tau^2 L^2 n \frac{1}{n}  {\bE\left\|P_{\teta g}(\ox^{r+1})-P_{\teta g}(\ox^r)\right\|^2}\\
    \le& \frac{1}{n}\cdot4 \eta^2 \tau L^2n\tau   \frac{1}{n\tau}{\sum_{t=0}^{\tau-1} \sum_{i=1}^n \mathbb{E}\underbrace{\left\|z_{i, t}^r-P_{\teta g}(\ox^r)\right\|^2}_{\le 2\| \hz_{i, t}^r- {P_{\teta g}(\ox^r)} \|^2+2 {t^2 \eta^2} B_g^2 }}
 + \frac{1}{n} {4} \eta^2  n^2\tau^2  \frac{\sigma^2}{n\tau b}\\
    \le&\frac{1}{n} {2\cdot4 \eta^2 \tau L^2n\tau} \frac{1}{n\tau} \Bigg({ {32} \tau\mathbb{E}\left\|\Lambda^r-\overline{\Lambda}^r\right\|^2 }+{32} \eta^2\tau^3 n  \mathbb{E}\left\| {\nabla f}(P_{\teta g}(\ox^r)) \right\|^2+{8} \eta^2 \tau^2 n\frac{\sigma^2}{b}+{\tau  {\frac{\teta^2}{\eta_g^2}} nB_g^2 {+ t^2 \eta^2 B_g^2}}\Bigg)\\
    &+ \frac{1}{n} {4} \eta^2  n^2\tau^2  \frac{\sigma^2}{n\tau b}.\\
\end{aligned}
\end{equation}

Similarly, substituting \eqref{eq:phi--29}  into {\eqref{eq:x-bar-x*-24}},  we have 
\begin{equation}\label{eq:x-x*-35}
\begin{aligned}
    &\mathbb{E}\left\|P_{\teta g}(\ox^{r+1})-x^{\star}\right\|^2- {2\teta\left( \frac{L}{2}- \frac{1}{4\teta}\right)\bE\|P_{\teta g}(\ox^{r+1})-P_{\teta g}(\ox^{r})\|^2 } \\
    \le&	\left(1-\frac{\mu\teta}{2} \right) \bE\|P_{\teta g}(\ox^{r})-x^{\star} \|^2
    +\left(4\teta^2 L^2 +\frac{2\teta L^2}{\mu}\right) \frac{1}{n\tau} \sum_{i=1}^{n} \sum_{t=0}^{\tau-1}\underbrace{ {\|z_{i, t}^r-P_{\teta g}(\ox^r)\|^2}  }_{\le 2\| \hz_{i, t}^r- {P_{\teta g}(\ox^r)} \|^2+2 {t^2 \eta^2} B_g^2 }+ 4\teta^2 \frac{1}{n\tau} \frac{\sigma^2}{b}\\
    \le& \left(1-\frac{\mu\teta}{2} \right) \|P_{\teta g}(\ox^{r})-x^{\star} \|^2 + 4\teta^2 \frac{1}{n\tau} \frac{\sigma^2}{b}\\
    &+ \left(4\teta^2 L^2 +\frac{2\teta L^2}{\mu}\right)\frac{2}{n\tau}\left({ {32} \tau\mathbb{E}\left\|\Lambda^r-\overline{\Lambda}^r\right\|^2 }+{32} \eta^2\tau^3 n  \mathbb{E}\left\| {\nabla f}(P_{\teta g}(\ox^r)) \right\|^2+{8} \eta^2 \tau^2 n\frac{\sigma^2}{b}+(n\tau+1)  {\frac{\teta^2}{\eta_g^2}} B_g^2 \right).\\
\end{aligned}
\end{equation} 
Noting that $2\teta\left( \frac{L}{2}- \frac{1}{4\teta}\right) +2\eta^2\tau^2 L^2< 0   $ if $\eta \le \frac{1}{4\tau \eta_g L}$ and $\eta_g=\sqrt{n}$, summing \eqref{eq:Lambda-34} and \eqref{eq:x-x*-35} yields 
\begin{equation}\label{eqn:Lypnv}
\begin{aligned}
    &\bE[\Omega^{r+1}] \\
\le & \left(1-\frac{\mu\teta}{2}\right) \bE\left\|P_{\teta g}(\ox^r)-x^{\star}\right\|^2+ \underbrace{\left( 2\cdot4 \eta^2 \tau L^2n\tau  \frac{1}{n\tau}  32 \tau + \left(4\teta^2 L^2 +\frac{2\teta L^2}{\mu}\right) 64   \right) }_{\le 1-\frac{\mu\teta}{2} ~ if ~ { \teta\le \frac{\mu}{150 L^2} }}  \frac{1}{n}\mathbb{E}\left\|\Lambda^r\!-\!\overline{\Lambda}^r\right\|^2  + 6\teta^2  \frac{\sigma^2}{n\tau b}+4\teta^2 \frac{1}{n\tau} \frac{\sigma^2}{b}\\
    &+\left( \underbrace{\left( 2\cdot4 \eta^2 \tau L^2n\tau\frac{1}{n} \frac{1}{n\tau} \cdot 32  \eta^2\tau^3 n \right)}_{\le 256 \teta^4 \frac{L^2}{n^2}} + \underbrace{\left(4\teta^2 L^2 +\frac{2\teta L^2}{\mu}\right) \frac{2}{n\tau}{32} \eta^2\tau^3 n }_{\le \frac{2\teta^2}{n}} \right) \mathbb{E} \underbrace{\left\| {\nabla f}(P_{\teta g}(\ox^r)) \right\|^2}_{\le 2\left\| {\nabla f}(P_{\teta g}(\ox^r)) - {\nabla f}(x^{\star})\right\|^2+2\left\|  {\nabla f}(x^{\star}) \right\|^2  }\\
&+\left(\underbrace{\left(4\teta^2 L^2 +\frac{2\teta L^2}{\mu}\right)\frac{2}{n\tau}+\frac{8\teta^2L^2}{n^2\tau } }_{\le \frac{1}{10n\tau}} \right)\cdot (\tau n+1)\frac{\teta^2}{\eta_g^2}B_g^2 \\
    \le& \left(1-\frac{\mu \teta}{2}+ {\frac{6\teta^2}{n}L^2 } \right)\bE\|P_{\teta g}(\ox^r)-x^{\star}\|^2+ \left(1-\frac{\mu\teta}{2}\right)\frac{1}{n}\mathbb{E}\left\|\Lambda^r\!-\!\overline{\Lambda}^r\right\|^2 +10\teta^2  \frac{\sigma^2}{n\tau b}+ \frac{7\teta^2}{n} B_g^2\\
    \le& \left(1-\frac{\mu\teta}{3}\right)\bE[\Omega^{r}] +10\teta^2  \frac{\sigma^2}{n\tau b}+ \frac{7\teta^2}{n} B_g^2, 
\end{aligned}
\end{equation}
where we use \eqref{eq:stepsize},  $(\eta \tau L)^2\le \frac{1}{150^2}\frac{\mu^2}{L^2}\frac{1}{\eta_g^2} $, $\teta^2 L^2  \le \frac{1}{150^2}\frac{\mu^2}{L^2} $, $\frac{\teta L^2}{\mu }\le \frac{1}{150}$, and ${\mu\teta }\le \frac{1}{150}\frac{\mu^2}{L^2}$.  
Since $\nabla f(x^{\star})+\widetilde{\nabla} g(x^{\star}) =0 $, we have that $\|\nabla f(x^{\star})\|\le B_g$. 
After telescoping the above inequality, we complete the proof of Theorem \ref{thm: g}.  
\subsection{Proof of Corollary \ref{coro-Ic}}
\begin{proof}
In the special case when $g=I_{\mC}$, if $\eta \le \frac{1}{4L\tau}$, we have 
\begin{equation}
    \begin{aligned}
        &{4} \eta^2 \tau\left\| \nabla \mathbf{f}\left(\bz_t^r \right)-\nabla \mathbf{f}\left(P_{\teta g}(\obx^r)\right)\right\|^2\\
        =&{4} \eta^2 \tau\|\nabla \mathbf{f} (P_{t\eta  g}(\hbz_t^r)) -\nabla \mathbf{f} (P_{t\eta g }(P_{\teta g}(\obx^r)))\|^2\\
        \le&\frac{1}{4\tau}\|\hbz_t^r -P_{\teta g}(\obx^r) \|^2,
    \end{aligned}
\end{equation}
where we use the fact that $P_{\teta g}(\obx^r)=P_{t\eta g}(P_{\teta g}(\obx^r)) $ when $g=I_{\mathcal{C}}$   in the equality. 
Thus, \eqref{eq-phi-P1} becomes 
\begin{equation}
    \begin{aligned}
        &\mathbb{E}\left[ \sum_{i=1}^n\left\|\hz_{i, t+1}^r- {P_{\teta g}(\ox^r)}\right\|^2\right]       \leq\left(\!1+\frac{5 / 4}{\tau\!-\!1}\!\right) \mathbb{E}\left\|\hbz_t^r\!-\! {P_{\teta g}(\obx^r)} \right\|^2\!+\!\frac{ {4}\bE\left\|\Lambda^r\right\|^2 }{\tau }\!+ {2}\!\eta^2 n \frac{\sigma^2}{b}
    \end{aligned}
\end{equation}
and \eqref{eq:phi--29} becomes 
\begin{equation}\label{eq:phi--29-IC}
    \begin{aligned}
              &{\mathbb{E}\left[\sum_{t=0}^{\tau-1} \sum_{i=1}^n\left\|\hz_{i, t}^r- {P_{\teta g}(\ox^r)}\right\|^2\right]}
        \le	
        { {32} \tau\mathbb{E}\left\|\Lambda^r-\overline{\Lambda}^r\right\|^2 }+{32} \eta^2\tau^3 n  \mathbb{E}\left\| {\nabla f}(P_{\teta g}(\ox^r)) \right\|^2+{8} \eta^2 \tau^2 n\frac{\sigma^2}{b}.
    \end{aligned}
\end{equation}
In addition, by $\left\|z_{i, t}^r-P_{\teta g}(\ox^r)\right\|^2=\left\|P_{t\eta g}(\hz_{i, t}^r)-P_{t\eta g}(P_{\teta g}(\ox^r))\right\|^2\le \| \hz_{i, t}^r- {P_{\teta g}(\ox^r)} \|^2$, 
\eqref{eq:Lambda-34} becomes 
\begin{equation}\label{eq:Lambda-34-IC}
\begin{aligned}
    &\frac{1}{n}\bE\left\|\Lambda^{r+1}-\overline{\Lambda}^{r+1}\right\|^2-2\eta^2\tau^2 L^2 n \frac{1}{n}  {\bE\left\|P_{\teta g}(\ox^{r+1})-P_{\teta g}(\ox^r)\right\|^2}\\
    \le&\frac{1}{n} {2\cdot4 \eta^2 \tau L^2n\tau} \frac{1}{n\tau} \Bigg({ {32} \tau\mathbb{E}\left\|\Lambda^r-\overline{\Lambda}^r\right\|^2 }+{32} \eta^2\tau^3 n  \mathbb{E}\left\| {\nabla f}(P_{\teta g}(\ox^r)) \right\|^2+{8} \eta^2 \tau^2 n\frac{\sigma^2}{b}\Bigg)  + \frac{1}{n} {4} \eta^2  n^2\tau^2  \frac{\sigma^2}{n\tau b}
\end{aligned}
\end{equation}
and \eqref{eq:x-x*-35} becomes
\begin{equation}\label{eq:x-x*-35-IC}
\begin{aligned}
    &\mathbb{E}\left\|P_{\teta g}(\ox^{r+1})-x^{\star}\right\|^2- {2\teta\left( \frac{L}{2}- \frac{1}{4\teta}\right)\bE\|P_{\teta g}(\ox^{r+1})-P_{\teta g}(\ox^{r})\|^2 } \\
    \le& \left(1-\frac{\mu\teta}{2} \right) \bE\|P_{\teta g}(\ox^{r})-x^{\star} \|^2 + 4\teta^2 \frac{1}{n\tau} \frac{\sigma^2}{b}\\
    &+ \left(4\teta^2 L^2 +\frac{2\teta L^2}{\mu}\right)\frac{2}{n\tau}\left({ {32} \tau\mathbb{E}\left\|\Lambda^r-\overline{\Lambda}^r\right\|^2 }+{32} \eta^2\tau^3 n  \mathbb{E}\left\| {\nabla f}(P_{\teta g}(\ox^r)) \right\|^2+{8} \eta^2 \tau^2 n\frac{\sigma^2}{b} \right).
\end{aligned}
\end{equation} 

Combing \eqref{eq:Lambda-34-IC} and \eqref{eq:x-x*-35-IC}, and  using $\nabla f(x^{\star})=0 $ given in Assumption \ref{asmp:Ic}, \eqref{eqn:Lypnv} becomes 
\begin{equation}
    \begin{aligned}
        \bE[\Omega^{r+1}]    \le& \left(1-\frac{\mu\teta}{3}\right)\bE[\Omega^{r}] +10\teta^2  \frac{\sigma^2}{n\tau b}.  
    \end{aligned}
\end{equation}
Thus, we complete the proof of Corollary \ref{coro-Ic}. 
\end{proof}
\end{document}